\documentclass[11pt]{article}

\usepackage{acl}

\usepackage{times}
\usepackage{latexsym}
\usepackage[T1]{fontenc}
\usepackage[utf8]{inputenc}
\usepackage{microtype}
\usepackage{inconsolata}
\usepackage{graphicx}
\usepackage{booktabs}
\usepackage{colortbl}
\usepackage{amsmath}
\usepackage{amssymb}
\usepackage{amsthm}
\usepackage{algorithm}
\usepackage{algorithmic}
\usepackage{array}
\usepackage{placeins}

\newtheorem{proposition}{Proposition}

\title{MileGPO: Milestone Inference with Local Evidence for Graph-Based Policy Optimization of Long-Horizon LLM Agents}

\author{
	Bo Qian, 
	Yuting Wu\thanks{\;Corresponding author.},  
	Shuang Zeng, Huaiyu Wan, Dalin Zhang \and
	Jiqiang Liu \\
	Beijing Jiaotong University, China\\
	\tt \{ytwu1,hywan,dalin,jqliu\}@bjtu.edu.cn \\
    \tt bobo1398861921@gmail.com, zengs@pku.edu.cn \\
}

\begin{document}
\maketitle

\begin{abstract}
Credit assignment is challenging in long-horizon agentic reinforcement
learning, where supervision often comes only from final rewards. Existing
methods refine trajectory-level signals into step-level credits through step
grouping or graph-based advantage estimation, but can overlook meaningful
intermediate milestones. We propose \textbf{MileGPO} (\underline{M}ilestone
\underline{I}nference with \underline{L}ocal \underline{E}vidence for
\underline{G}raph-Based \underline{P}olicy \underline{O}ptimization), which
derives process-level credit from grouped on-policy rollouts through three
designs. \textbf{Milestone Discovery} identifies candidate milestones on
successful rollouts and recurring traps on failed ones.
\textbf{Reliability-Calibrated Shaping (RCS)} weights these candidates by
outcome-based confidence, strengthening reliable milestones and traps while
down-weighting uncertain ones. \textbf{Progress-Contrastive Calibration
(PCC)} further tests whether a candidate reflects local progress and whether
its incoming transition outperforms observed alternatives from the same state.
MileGPO requires neither auxiliary models nor additional environment
interaction. Experiments on ALFWorld and WebShop show state-of-the-art
performance and a small in-distribution to out-of-distribution gap on
ALFWorld. Ablations and credit diagnostics indicate that reliability
weighting, local progress, and same-state branch evidence complement milestone
discovery and resolve ambiguous intermediate credit.
\end{abstract}

\section{Introduction}

\begin{figure}[t]
\centering
\includegraphics[width=\columnwidth]{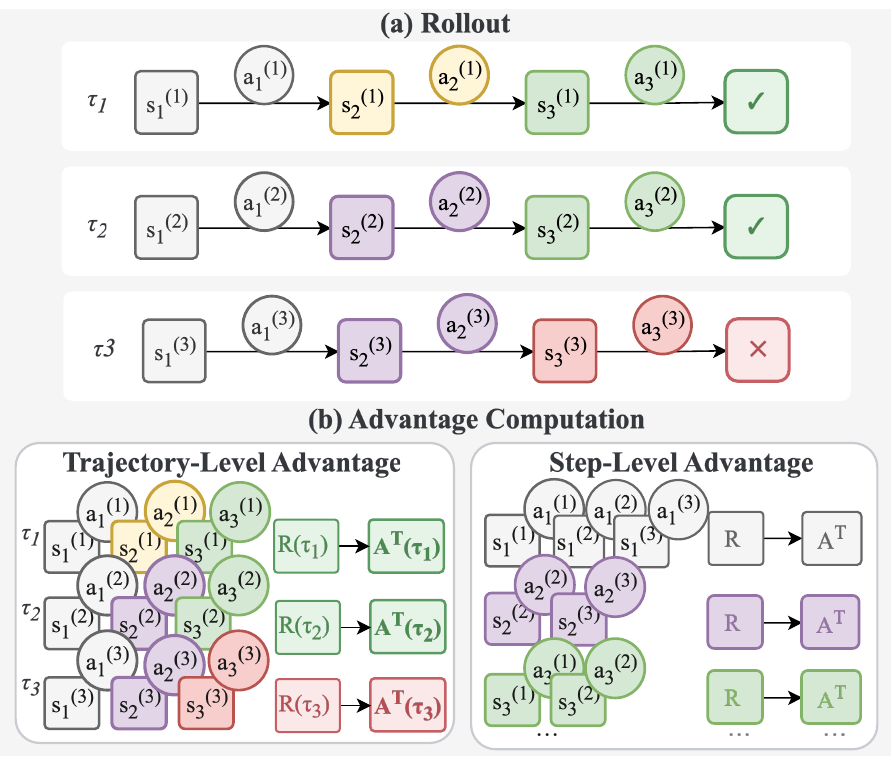}
\caption{Trajectory-level and same-state step-level credit for two successful rollouts and one failed rollout. Here, $\tau_i$ is rollout $i$; $s_t^{(i)}$, $a_t^{(i)}$, $R(\tau_i)$, and $A^{\mathrm{T}}(\tau_i)$ denote its state, action, final reward, and trajectory-level advantage. Matching colors mark the same state across rollouts.}
\label{fig:credit-assignment-overview}
\end{figure}

In recent years, large language models (LLMs) have evolved from single-turn text generators into capable agents that can reason, invoke external tools, and execute tasks over extended horizons \cite{yao2023reactsynergizingreasoningacting,schick2023toolformer,wang2023voyager}. Such agents are increasingly applied to long-horizon tasks such as web navigation and embodied instruction following \cite{deng2023mind2web,zhou2024webarena,shridhar2021alfworldaligningtextembodied}, where successful completion depends on a coherent sequence of interdependent decisions.
Reinforcement learning (RL) provides a natural framework for learning to improve such agents from final rewards at the task-level \cite{ouyang2022traininglanguagemodelsfollow,deepseekai2025deepseekr1,wang2025ragenunderstandingselfevolutionllm}. However, when supervision is reduced to a final reward, assigning credit to intermediate decisions becomes challenging. As illustrated in Figure~\ref{fig:credit-assignment-overview}, trajectory-level methods assign the same final-reward-derived advantage to all actions in a rollout, while step-level grouping methods compare actions taken under the same state and provide finer-grained credit \cite{williams1992simple,schulman2017proximalpolicyoptimizationalgorithms,feng2025groupingrouppolicyoptimizationllm}. GraphGPO \cite{cheng2026trajectorylevelattributiongraphbasedcredit} further organizes grouped on-policy rollouts into a task-local transition graph, where each transition represents a state--action--next-state interaction step, and propagates credit according to each state's shortest-path distance to a successful final state. Although this graph structure provides dense intermediate supervision, final goal distance primarily captures \emph{reachability}, rather than \emph{reliable progress}: states with the same distance can have substantially different probabilities of leading to success, while important intermediate stages may receive weak credit simply because they remain far from the final goal.

\begin{figure}[t]
\centering
\includegraphics[width=\columnwidth]{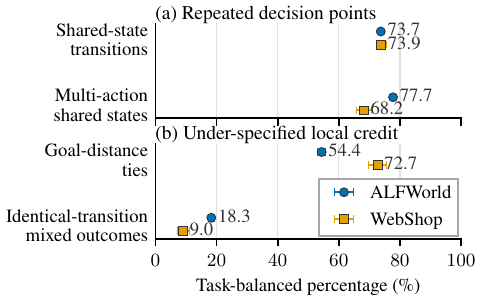}
\caption{Structural ambiguity in early GraphGPO rollouts. Shared states create opportunities for local branch comparison, but final-state-distance ties and mixed outcomes leave progress ambiguous.}
\label{fig:motivation-metrics}
\end{figure}

As shown in Figure~\ref{fig:motivation-metrics}, in ALFWorld and WebShop, about $74\%$ of the transitions originate from states shared among the rollouts, and most shared states contain multiple observed actions, providing recurring states and sibling branches from which to infer intermediate credit. However, the final-goal distance assigns the same credit to $54.4\%$ and $72.7\%$ of the same-state action pairs, respectively, while identical transitions can appear in both successful and failed trajectories. Thus, the rollout graph offers useful evidence for finer credit assignment, but exploiting it raises three questions about how to discover, weight, and validate intermediate credit anchors.

To address these questions, we propose the \textbf{MileGPO} (\underline{M}ilestone \underline{I}nference with \underline{L}ocal \underline{E}vidence for \underline{G}raph-Based \underline{P}olicy \underline{O}ptimization), illustrated in Figure~\ref{fig:method-overview}, with a sequence of nested credit designs: (1) \textbf{Milestone Discovery (MD)}: Not every state on a successful trajectory marks meaningful progress. MD uses final rewards and recurring rollout structures to identify candidate milestones from successful trajectories and recurring traps from failed ones. (2) \textbf{Reliability-Calibrated Shaping (RCS)}: Candidate anchors have unequal final-reward evidence. RCS weights their positive or negative influence by empirical reliability, strengthening well-supported candidates while suppressing uncertain ones. (3) \textbf{Progress-Contrastive Calibration (PCC)}: Final-reward association alone cannot distinguish genuine progress from incidental correlation. PCC evaluates local advancement, while its \textbf{Branch-Counterfactual Credit (BCC)} component compares sibling transitions from the same state. MileGPO modifies only step-level advantage estimation and requires no external process annotations or auxiliary inference.

We evaluate MileGPO on two challenging long-horizon agent benchmarks, ALFWorld and WebShop. MileGPO achieves state-of-the-art performance on both benchmarks, with consistent improvements over reproduced graph-based baselines. Its ALFWorld ID--OOD gap is only $1.69$ points, smaller than those of GiGPO ($1.89$) and GraphGPO ($3.78$), demonstrating stronger generalization to unseen task configurations. Further analysis shows that milestone discovery alone can introduce noisy credit, whereas reliability weighting recovers outcome-consistent preferences and local progress with the same-state branch evidence resolves distinctions obscured by final-goal distance. These findings attribute MileGPO's gains to calibrated intermediate credit rather than to indiscriminate milestone propagation.

Our main contributions are as follows:
\begin{itemize}
    \item \textbf{Revealing unreliable intermediate credit.} We reveal that final-goal-distance credit leaves many same-state branches indistinguishable and that success-visited states do not necessarily represent reliable progress.
    \item \textbf{Proposing a rollout-native policy optimization algorithm.} We introduce MileGPO, which learns from rollout graphs and on-policy rewards without external annotations, critics, reward models, or auxiliary inference.
    \item \textbf{Achieving strong empirical performance.} MileGPO achieves state-of-the-art performance on ALFWorld and WebShop, while ablations and diagnostics validate its credit-calibration mechanism.
\end{itemize}

\section{Related Work}

\subsection{Reinforcement Learning for LLM Agents}

LLM post-training builds on policy-gradient estimators such as REINFORCE, GAE, and PPO
\cite{williams1992simple,schulman2018highdimensionalcontinuouscontrolusing,schulman2017proximalpolicyoptimizationalgorithms}.
RLHF learns human-preference rewards
\cite{stiennon2022learningsummarizehumanfeedback,ouyang2022traininglanguagemodelsfollow,bai2022traininghelpfulharmlessassistant},
whereas DPO bypasses an explicit reward-model RL loop
\cite{rafailov2024directpreferenceoptimizationlanguage}.
For verifiable reasoning, RLOO, GRPO, and DAPO compare sampled responses without a value model
\cite{kool2019buy4reinforce,ahmadian2024backbasics,shao2024deepseekmathpushinglimitsmathematical,yu2025dapoopensourcellmreinforcement},
as exemplified by DeepSeek-R1, while process supervision adds intermediate feedback
\cite{deepseekai2025deepseekr1,lightman2023letsverifystepstep}.

These objectives extend to agents that interact with search engines, websites, embodied environments, and other agents over multiple turns
\cite{zhou2025sweetrltrainingmultiturnllm}.
Reflection, search-guided collection, and hierarchical actor--critic methods use verbal memory, tree search, or learned values
\cite{shinn2023reflexion,putta2024agentq,zhou2024archer},
while Agent Lightning decouples agent execution from training
\cite{luo2025agentlightningtrainai}.
Recent methods refine credit granularity: GiGPO compares actions at recurring states, HGPO conditions comparisons on interaction history
\cite{he2026hierarchyofgroupspolicyoptimizationlonghorizon},
StepPO, BiPACE, and Progress Advantage use step-aligned, action-conditioned, or policy-derived signals
\cite{wang2026steppostepalignedpolicyoptimization,wang2026bipacebisimulationguidedpolicyoptimization,oh2026neglectedfreelunchposttraining},
and graph-based methods merge grouped trajectories into task-local structures
\cite{wang2026groupgraphpolicyoptimizationlonghorizon}.

\begin{figure*}[t]
\centering
\includegraphics[width=\textwidth]{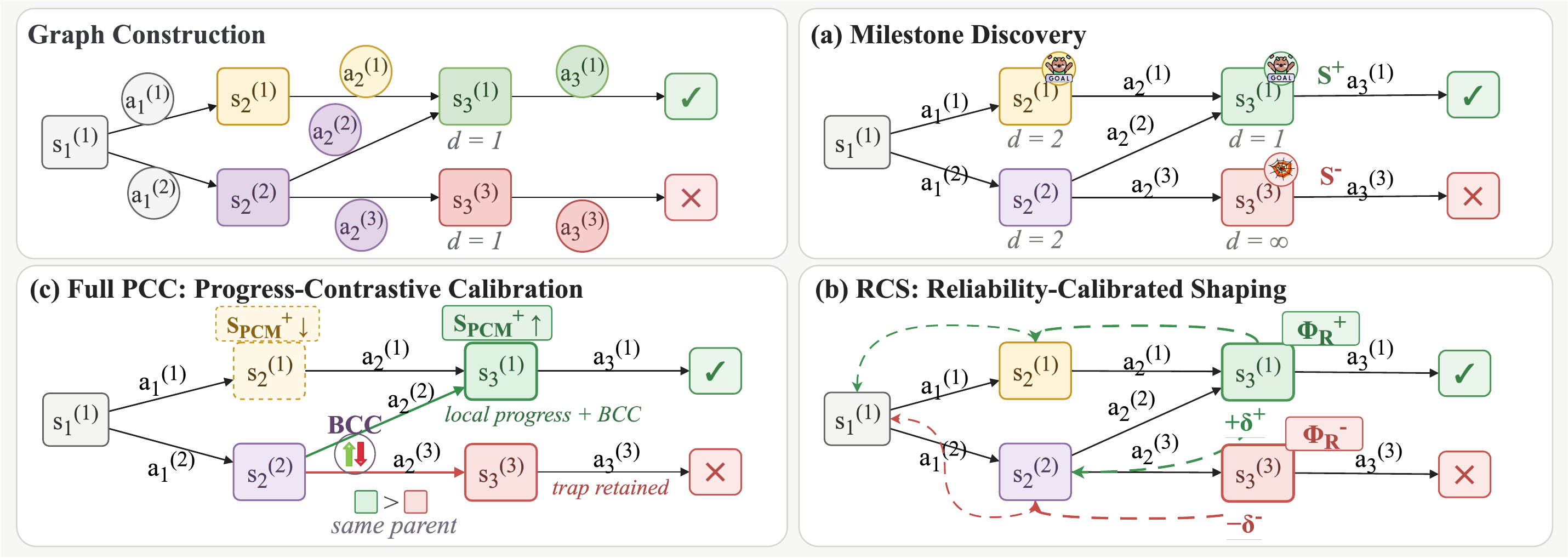}
\caption{Overview of MileGPO. Squares denote canonicalized states, circles denote sampled actions, and check and cross marks indicate success and failure. Here, $d$ is directed shortest-path distance ($\infty$ if unreachable); $\mathcal{S}^{+}$ and $\mathcal{S}^{-}$ are the milestone and trap sets; $\Phi_{\mathrm{R}}^{+}$ and $\Phi_{\mathrm{R}}^{-}$ are their score-weighted potentials; and $\delta^{+}$ and $\delta^{-}$ are the corresponding increases along a transition. MD discovers candidates, RCS weights and propagates them, and PCC adjusts positive scores using local progress and BCC comparisons between branches from the same state.}
\label{fig:method-overview}
\end{figure*}

\subsection{Process-Level Credit Assignment}

Process reward models provide step-level supervision but require process labels and may degrade under distribution shift
\cite{lightman2023letsverifystepstep,wang2024mathshepherd,zhang2025lessonsprm}.
Annotation-light alternatives estimate intermediate values through search or derive process signals from outcome rewards
\cite{chen2024alphamath,cui2025prime}.
TreeRPO and TreeRL share prefixes and compare branches through tree-structured rollouts, while GraphPO merges equivalent reasoning states during structured sampling
\cite{yang2025treerpo,hou2506treerl,zhan2026graphpo}.
These approaches introduce auxiliary process estimation or alter rollout collection through search and branching.

\section{Method}
\label{sec:method}

MileGPO adds intermediate credit to the final-goal credit used by GraphGPO. It has three steps. MD finds milestone and trap candidates from grouped on-policy rollouts. RCS weights these candidates by their scores and converts changes in graph-distance potentials into positive or negative credit. PCC rechecks positive candidates by measuring local progress and comparing transitions from the same source state. Figure~\ref{fig:method-overview} shows the full procedure. All three steps reuse one rollout graph and change only the advantage estimation.

\subsection{Preliminaries: Grouped Rollout Graph}

For each task $q$, the current policy samples a group of $K$ interaction trajectories $\mathcal{T}_q=\{\tau_i\}_{i=1}^{K}$. Each trajectory $\tau_i=(s_{i,0},a_{i,0},\ldots,s_{i,T_i})$ has length $T_i$ and a final environment reward $R_i^{\mathrm{env}}$, and $y_i=\mathbf{1}[R_i^{\mathrm{env}}>\theta_R]$ marks success under the task-specific threshold $\theta_R$. Following GraphGPO \cite{cheng2026trajectorylevelattributiongraphbasedcredit}, we merge all trajectories for task $q$ into a directed transition graph $\mathcal{G}_q=(\mathcal{V}_q,\mathcal{E}_q)$ whose nodes are canonicalized observations and whose edges $e=(u,v)$ are observed transitions. We write $g$ for the successful final state and $d(u,v)$ for the directed shortest-path distance, setting $d(u,v)=\infty$ when $v$ is unreachable, in which case every distance-decayed term below is zero.

GraphGPO then assigns each observed transition a dense return according to the distance from its destination to $g$:
\begin{equation}
 r^{\mathrm{G}}(u,v)=c\,\gamma_G^{d(v,g)},
 \label{eq:graphgpo-return}
\end{equation}
% [OLD] where $c$ sets the return scale and $\gamma_G\in(0,1]$ controls how quickly the return decays with distance. Transitions with the same source state are later normalized as one comparison group. This return provides final-goal reachability, but it cannot distinguish between two destinations at the same distance. MileGPO preserves this graph signal and supplements it with intermediate credit.
where $c$ sets the return scale and $\gamma_G\in(0,1]$ controls decay with distance. This return captures reachability but cannot distinguish destinations at the same graph distance; \emph{MileGPO preserves it and supplements it with intermediate credit as shown in Figure~\ref{fig:method-overview}.} 

\subsection{MD: Milestone Discovery}

MD uses the outcome rewards and transition structure of the rollout group to discover two kinds of intermediate targets.
A \emph{milestone candidate} is a nonfinal state visited by successful trajectories. A \emph{trap candidate} is a state observed only in failed trajectories that either appears across multiple failures or is revisited. MD first scores these states and then propagates every selected state with equal weight.

\subsubsection{Finding Candidates}

% [OLD] We count distinct trajectories rather than raw visits so that a loop within one trajectory cannot dominate the statistics. Let $\mathcal{T}_q^{+}=\{\tau_i:y_i=1\}$ and $\mathcal{T}_q^{-}=\{\tau_i:y_i=0\}$ be the successful and failed trajectories for task $q$. For a node $v$, let $\mathcal{T}(v)$ contain the trajectories that visit $v$, with $\mathcal{T}^{+}(v)$ and $\mathcal{T}^{-}(v)$ denoting its successful and failed subsets. We compute the overall success rate of the rollout group and the success rate conditioned on visiting $v$ as
Let $\mathcal{T}_q^{+}=\{\tau_i:y_i=1\}$ and $\mathcal{T}_q^{-}=\{\tau_i:y_i=0\}$ be the successful and failed trajectories for task $q$. For a node $v$, let $\mathcal{T}(v)$ contain the trajectories that visit $v$ (each counted once), with $\mathcal{T}^{+}(v)$ and $\mathcal{T}^{-}(v)$ denoting its successful and failed subsets. We compute the group and node-conditional success rates as
\begin{equation}
 p_q^{+}=\frac{|\mathcal{T}_q^{+}|}{|\mathcal{T}_q|}, \qquad
 p^{+}(v)=\frac{|\mathcal{T}^{+}(v)|}{|\mathcal{T}(v)|}.
\end{equation}

To score a milestone candidate $v$, MD first compares its conditional success rate $p^{+}(v)$ with the rollout-group success rate $p_q^{+}$. We retain only positive differences:
\begin{equation}
 \ell^{+}(v)=\max\left(p^{+}(v)-p_q^{+},0\right).
\end{equation}
Thus, $\ell^{+}(v)>0$ only when trajectories visiting $v$ have a higher success rate than the average across the rollout group. We separately measure how broadly $v$ appears across successful trajectories:

% We separately measure how broadly $v$ appears across successful trajectories:
\begin{equation}
 m(v)=\frac{|\mathcal{T}^{+}(v)|}{\max(|\mathcal{T}_q^{+}|,1)}.
 \label{eq:mandatory}
\end{equation}
% [OLD] Here, $m(v)$ is the fraction of sampled successful trajectories that visit $v$. We also define $C(v)$ as the max-normalized sum of the in-degree and out-degree of $v$ among success-visited states. MD combines these three quantities into the initial milestone score:
Here, $m(v)$ is the fraction of sampled successful trajectories that visit $v$. We also define $C(v)$ as the sum of the in-degree and out-degree of $v$ in $\mathcal{G}_q$, divided by the largest such sum over success-visited states, so that $C(v)\in[0,1]$ measures how central $v$ is among the states that successful trajectories traverse. MD combines these three quantities into the initial milestone score:
% We also define $C(v)$ as the max-normalized sum of the in-degree and out-degree of $v$ among success-visited states. MD combines these quantities into the initial milestone score:
\begin{equation}
 S_0^{+}(v)=w_s\ell^{+}(v)+w_m m(v)+w_c C(v).
 \label{eq:base-milestone-score}
\end{equation}
$w_s$, $w_m$, and $w_c$ weight the success-rate difference, successful trajectory coverage, and graph connectivity, respectively.

% [OLD] MD uses failed trajectories to find states that should receive negative credit. Let $p_q^{-}=1-p_q^{+}$ be the group failure rate
% [OLD-v2] Let $p_q^{-}=1-p_q^{+}$ be the group failure rate and let $p^{-}(v)=|\mathcal{T}^{-}(v)|/|\mathcal{T}(v)|$ be the failure rate among trajectories that visit $v$. For a state that no successful trajectory visits, we keep only the amount by which this rate exceeds the group average: $\ell^{-}(v)=\max(p^{-}(v)-p_q^{-},0)$. We also compute the fraction of failed trajectories that visit $v$ as $f^{-}(v)=|\mathcal{T}^{-}(v)|/\max(|\mathcal{T}_q^{-}|,1)$. Finally, $L(v)=r(v)/\max(|\mathcal{T}(v)|,1)$ measures repeated visits, where $r(v)$ counts visits to $v$ after its first occurrence in each trajectory. MD combines these quantities into the trap score
Let $p_q^{-}=1-p_q^{+}$ be the group failure rate and let $p^{-}(v)=|\mathcal{T}^{-}(v)|/|\mathcal{T}(v)|$ be the failure rate among trajectories that visit $v$. For a state that no successful trajectory visits, we keep only the amount by which this rate exceeds the group average: $\ell^{-}(v)=\max(p^{-}(v)-p_q^{-},0)$. We also compute the fraction of failed trajectories that visit $v$ as $f^{-}(v)=|\mathcal{T}^{-}(v)|/\max(|\mathcal{T}_q^{-}|,1)$. Finally, $L(v)=n_{\mathrm{re}}(v)/\max(|\mathcal{T}(v)|,1)$ is the average number of revisits per visiting trajectory, where $n_{\mathrm{re}}(v)$ counts visits to $v$ after its first occurrence in each trajectory. MD combines these quantities into the trap score: 
\begin{equation}
 S^{-}(v)=w_f\ell^{-}(v)+w_l f^{-}(v)L(v).
 \label{eq:trap-score}
\end{equation}
% The two terms favor states that are associated with failure and recur in failed behavior. We denote the selected milestones and traps by $\mathcal{S}^{+}$ and $\mathcal{S}^{-}$, respectively.
% [OLD] We denote the selected milestones and traps by $\mathcal{S}^{+}$ and $\mathcal{S}^{-}$, respectively.
where $w_f$ and $w_l$ weight the excess failure rate and the recurrence of $v$ in failed behavior. We denote the selected milestones and traps by $\mathcal{S}^{+}$ and $\mathcal{S}^{-}$, respectively.

\subsubsection{Uniform target propagation}
MD assigns every selected state unit weight and propagates it through the graph distance. With decay $\omega\in(0,1]$, the two potential maps are
\begin{equation}
 \begin{array}{rcl}
 \Phi_{\mathrm{S}}^{+}(s)&=&\displaystyle\max_{v\in\mathcal{S}^{+}}
 \omega^{d(s,v)}, \\
 \Phi_{\mathrm{S}}^{-}(s)&=&\displaystyle\max_{v\in\mathcal{S}^{-}}
 \omega^{d(s,v)}.
 \end{array}
 \label{eq:uniform-milestone-potential}
\end{equation}
% [OLD] For each state $s$, the maximum retains the largest distance-decayed target value. RCS next replaces the unit weights with candidate scores.
% [OLD-v2] RCS next replaces the unit weights with candidate scores.
Here, $\Phi_{\mathrm{S}}^{+}(s)$ and $\Phi_{\mathrm{S}}^{-}(s)$ measure how close $s$ is to the nearest milestone and the nearest trap, respectively, with the maximum retaining the largest distance-decayed target value. RCS next replaces the unit weights with candidate scores.

\subsection{RCS: Reliability-Calibrated Shaping}

RCS uses the same candidates as MD but weights them by their task-wise max-normalized scores $\overline S_0^{+}$ and $\overline S^{-}$:
\begin{equation}
 \begin{array}{rcl}
 \Phi_{\mathrm{R}}^{+}(s)&=&\displaystyle\max_{v\in\mathcal{S}^{+}}
 \overline S_{0}^{+}(v)\omega^{d(s,v)}, \\
 \Phi_{\mathrm{R}}^{-}(s)&=&\displaystyle\max_{v\in\mathcal{S}^{-}}
 \overline S^{-}(v)\omega^{d(s,v)}.
 \end{array}
 \label{eq:reliability-calibrated-potential}
\end{equation}
% [OLD] Higher-scoring targets therefore produce larger potentials. We keep milestones and traps separate because their contributions have opposite signs.
Thus $\Phi_{\mathrm{R}}^{+}(s)$ and $\Phi_{\mathrm{R}}^{-}(s)$ are the score-weighted counterparts of $\Phi_{\mathrm{S}}^{+}(s)$ and $\Phi_{\mathrm{S}}^{-}(s)$, so that higher-scoring targets produce larger potentials. We keep milestones and traps separate because their contributions have opposite signs.
% [OLD] RCS next checks whether a transition increases either potential. For a transition from $s_t$ to $s_{t+1}$, it discards negative changes:
RCS next checks whether a transition increases either potential. Writing $\Phi^{+}$ and $\Phi^{-}$ for the positive and negative potentials in use, which are $\Phi_{\mathrm{R}}^{+}$ and $\Phi_{\mathrm{R}}^{-}$ here, a transition from $s_t$ to $s_{t+1}$ discards negative changes:
\begin{equation}
 \delta_t^{+}=\max\left(
 \gamma_{\Phi}w_{+}\Phi^{+}(s_{t+1})
 -w_{+}\Phi^{+}(s_t),0\right),
 \label{eq:positive-shaping}
\end{equation}
\begin{equation}
 \delta_t^{-}=\max\left(
 \gamma_{\Phi}w_{-}\Phi^{-}(s_{t+1})
 -w_{-}\Phi^{-}(s_t),0\right).
 \label{eq:negative-shaping}
\end{equation}
% [OLD] Here, $w_{+}$ and $w_{-}$ scale the changes, and $\gamma_{\Phi}$ discounts the next-state potential. The return uses $\delta_t^{+}-\delta_t^{-}$, rewarding movement toward milestones and penalizing movement toward traps. This is a one-sided credit rule rather than policy-invariant potential shaping.
Here, $\delta_t^{+}$ and $\delta_t^{-}$ are the discounted increases in milestone and trap potential along the transition, $w_{+}$ and $w_{-}$ scale these changes, and $\gamma_{\Phi}$ discounts the next-state potential. The return uses $\delta_t^{+}-\delta_t^{-}$, rewarding movement toward milestones and penalizing movement toward traps. This is a one-sided credit rule rather than policy-invariant potential shaping.
% Here, $w_{+}$ and $w_{-}$ scale the changes, and $\gamma_{\Phi}$ discounts the next-state potential. This is a one-sided credit rule rather than policy-invariant potential shaping.

\subsection{PCC: Progress-Contrastive Calibration}

RCS uses group-level outcome rewards, so a state may score highly even when entering it makes little progress. PCC refines positive candidates with two transition-level scores. BCC compares transitions from the same source state, while local progress measures movement toward the goal and association with success. 
These scores update candidates rather than add rewards.

\subsubsection{Branch-Counterfactual Credit (BCC)}

At a source state $u$, grouped rollouts may contain several distinct outgoing edges. We treat each observed edge $e=(u,v)$ as one sampled branch from $u$. BCC compares the success rate of this branch with those of the other branches observed from the same state. Let $\mathcal{T}(e)$ contain the trajectories that use $e$, with $\mathcal{T}^{+}(e)$ and $\mathcal{T}^{-}(e)$ denoting its successful and failed subsets. Their rates are
% At a source state $u$, BCC compares each outgoing edge $e=(u,v)$ against the other observed branches from $u$. Let $\mathcal{T}(e)$ denote the trajectories using $e$, with $\mathcal{T}^{+}(e)$ and $\mathcal{T}^{-}(e)$ its successful and failed subsets. Their rates are
\begin{equation}
 p^{+}(e)=\frac{|\mathcal{T}^{+}(e)|}{|\mathcal{T}(e)|}, \qquad
 p^{-}(e)=\frac{|\mathcal{T}^{-}(e)|}{|\mathcal{T}(e)|}.
\end{equation}
Let $\mathcal{B}(u)$ contain these sampled branches. For $e_i\in\mathcal{B}(u)$, the BCC margin is
% [OLD] \begin{equation}
% [OLD]  \widetilde c_{\mathrm{bcc}}(e_i)=
% [OLD]  p^{+}(e_i)-\overline p^{+}_{-i}.
% [OLD]  \label{eq:bcc}
% [OLD] \end{equation}
% [OLD] Here, $\overline p^{+}_{-i}$ is the mean success rate of the other transitions in $\mathcal{B}(u)$. We divide each margin by the largest absolute margin in task $q$:
\begin{equation}
 \widetilde c_{\mathrm{bcc}}(e_i)=
 p^{+}(e_i)-\mu^{+}(e_i).
 \label{eq:bcc}
\end{equation}
Here, $\mu^{+}(e_i)$ is the mean success rate of the branches in $\mathcal{B}(u)\setminus\{e_i\}$, that is, of the alternatives observed at the same source state
; we set $\widetilde c_{\mathrm{bcc}}(e_i)=0$ when $e_i$ is the only sampled branch. We divide each margin by the largest absolute margin in task $q$:
% We divide each margin by the largest absolute margin in task $q$:
\begin{equation}
 c_{\mathrm{bcc}}(e_i)=
 \frac{\widetilde c_{\mathrm{bcc}}(e_i)}
 {\max_{e\in\mathcal{E}_q}|\widetilde c_{\mathrm{bcc}}(e)|}.
\end{equation}
% [OLD] A positive score means that $e_i$ succeeds more often than the other branches observed from $u$. For each candidate $v$, we retain the largest positive score among successful incoming transitions:
% A positive score means that $e_i$ succeeds more often than the other branches observed from $u$. For each candidate $v$, we retain the largest positive score among successful incoming transitions:
\begin{equation}
 b(v)=\max_{\substack{e=(u,v),\ c_{\mathrm{bcc}}(e)>0\\
 \mathcal{T}^{+}(e)\neq\emptyset}} c_{\mathrm{bcc}}(e).
 \label{eq:branch-evidence}
\end{equation}
% [OLD] Thus, $b(v)$ records the strongest observed branch preference leading to $v$; it is not a causal-effect estimate.
Thus, $b(v)$ records the strongest observed branch preference leading to $v$, with $b(v)=0$ when no incoming transition satisfies both conditions; it is not a causal-effect estimate.

\subsubsection{Local Progress}

% [OLD] BCC compares branches but does not measure progress toward the goal. For $e=(u,v)$, we therefore combine the distance reduction $\Delta d(e)=d(u,g)-d(v,g)$, the success-rate difference $p^{+}(v)-p_q^{+}$, and the excess failure rate $m^{-}(e)=\max(p^{-}(e)-\overline p^{-}_{-e},0)$ relative to sibling transitions, where $\overline p^{-}_{-e}$ is the mean failure rate of the other branches in $\mathcal{B}(u)$. When no sibling exists, $m^{-}(e)$ uses the node-level difference $\max(p^{-}(v)-(1-p_q^{+}),0)$. The progress score is
% [OLD] \begin{equation}
% [OLD]  g_{\mathrm{pg}}(e)=
% [OLD]  \alpha_d\Delta d(e)
% [OLD]  +\alpha_s\left[p^{+}(v)-p_q^{+}\right]
% [OLD]  -\alpha_f m^{-}(e),
% [OLD]  \label{eq:progress-gain}
% [OLD] \end{equation}
BCC compares branches but does not capture goal progress. For $e=(u,v)$, we combine distance reduction $\Delta d(e)=d(u,g)-d(v,g)$, success-rate gain $p^{+}(v)-p_q^{+}$, and excess failure rate $\ell^{-}(e)=\max(p^{-}(e)-\mu^{-}(e),0)$, where $\mu^{-}(e)$ is the mean failure rate of sibling branches in $\mathcal{B}(u)$. Thus, $\ell^{-}(e)$ is the edge-level counterpart of $\ell^{-}(v)$ and falls back to $\max(p^{-}(v)-p_q^{-},0)$ when no sibling exists. The progress score is:
\begin{equation}
 \psi(e)=
 \alpha_d\Delta d(e)
 +\alpha_s\left[p^{+}(v)-p_q^{+}\right]
 -\alpha_f \ell^{-}(e),
 \label{eq:progress-gain}
\end{equation}
% [OLD] where $\alpha_d$, $\alpha_s$, and $\alpha_f$ weight the three terms. Let $\mathcal{E}_{\mathrm{pg}}^{+}(v)$ contain successful incoming transitions to $v$ with $g_{\mathrm{pg}}(e)>0$. PCC takes the largest score and normalizes it over the task graph:
% [OLD] \begin{equation}
% [OLD]  \widetilde p(v)=\max_{e\in\mathcal{E}_{\mathrm{pg}}^{+}(v)}
% [OLD]  g_{\mathrm{pg}}(e),\qquad
% [OLD]  p(v)=\frac{\widetilde p(v)}
% [OLD]  {\max_{z\in\mathcal{V}_q}\widetilde p(z)}.
% [OLD]  \label{eq:progress-evidence}
% [OLD] \end{equation}
% [OLD] The resulting $p(v)\in[0,1]$ measures the best successful incoming progress to $v$. Undefined distance changes, unavailable branch comparisons, and zero normalizers yield zero.
where $\alpha_d$, $\alpha_s$, and $\alpha_f$ weight the three terms. Let $\mathcal{E}_{\mathrm{pg}}^{+}(v)$ contain successful incoming transitions to $v$ with $\psi(e)>0$. PCC takes the largest score and normalizes it over the graph:
\begin{equation}
\begin{aligned}
 \widetilde\psi(v)&=\max_{e\in\mathcal{E}_{\mathrm{pg}}^{+}(v)}
 \psi(e),\\
 \overline\psi(v)&=\frac{\widetilde\psi(v)}
 {\max_{v'\in\mathcal{V}_q}\widetilde\psi(v')}.
\end{aligned}
 \label{eq:progress-evidence}
\end{equation}
Here, $\widetilde\psi(v)$ is the best progress score among successful incoming transitions and $\overline\psi(v)\in[0,1]$ is its task-wise max-normalized value. Undefined distance changes, empty maxima, and zero normalizers yield zero.

\subsubsection{Updating Candidate Scores}
PCC combines the branch score and the progress score:
\begin{equation}
 E(v)=w_{\mathrm{bc}}b(v)+w_{\mathrm{pg}}\overline\psi(v).
 \label{eq:pcc-evidence}
\end{equation}
% [OLD] The weights $w_{\mathrm{bc}}$ and $w_{\mathrm{pg}}$ set their contributions. We set $z(v)=1$ when $p(v)>0$, $m(v)\geq\tau_m$, or, when branch retention is enabled, $b(v)>0$; otherwise, $z(v)=0$. BCC still contributes to $E(v)$ when another condition retains the candidate. PCC updates the score as
% [OLD-v2] We set $z(v)=1$ when $p(v)>0$, $m(v)\geq\tau_m$, or, when branch retention is enabled, $b(v)>0$; otherwise, $z(v)=0$. BCC still contributes to $E(v)$ when another condition retains the candidate. PCC updates the score as
% [OLD-v3] where $w_{\mathrm{bc}}$ and $w_{\mathrm{pg}}$ set their contributions. A retention indicator $z(v)$ marks candidates with local support: $z(v)=1$ when $p(v)>0$, when coverage $m(v)$ reaches a threshold $\tau_m$, or, when branch retention is enabled, when $b(v)>0$; otherwise, $z(v)=0$. BCC contributes to $E(v)$ when another condition retains the candidate. PCC updates the score as:
where $w_{\mathrm{bc}}$ and $w_{\mathrm{pg}}$ set their contributions. A retention indicator $z(v)$ marks candidates with local support: $z(v)=1$ when $\overline\psi(v)>0$, when coverage $m(v)$ reaches a threshold $\theta_m$, or when $\kappa_{\mathrm{bc}}=1$ and $b(v)>0$; otherwise, $z(v)=0$. The binary switch $\kappa_{\mathrm{bc}}$ controls whether branch evidence alone may retain a candidate, and BCC contributes to $E(v)$ regardless of which condition retains it. PCC updates the score as:
\begin{equation}
 S_{\mathrm{pcc}}^{+}(v)=
 \begin{cases}
 S_0^{+}(v)\left[1+w_{\mathrm{pcc}}E(v)\right], & z(v)=1,\\
 \rho S_0^{+}(v), & z(v)=0.
 \end{cases}
 \label{eq:pcc-reweight}
\end{equation}
Here, $w_{\mathrm{pcc}}$ amplifies retained candidates according to $E(v)$, while $\rho\in[0,1]$ shrinks the others. Max-normalization within task $q$ then produces $\overline S_{\mathrm{pcc}}^{+}$.

\begin{table*}[t]
\centering
\begingroup
% 使用正文 10 pt 字号；轻微收紧基础列间距后再弹性铺满双栏宽度。
\resizebox{\textwidth}{!}{%
\begin{tabular}{@{}l|ccccccc|cc@{}}
\toprule
\textbf{Method} & \multicolumn{7}{c|}{\textbf{ALFWorld}} & \multicolumn{2}{c}{\textbf{WebShop}} \\
\cmidrule(lr){2-8}\cmidrule(lr){9-10}
 & \textbf{Pick} & \textbf{Clean} & \textbf{Cool} & \textbf{Look} & \textbf{Heat} & \textbf{Pick2} & \textbf{All} & \textbf{Score} & \textbf{Succ.} \\
\midrule
\multicolumn{10}{@{}l}{\textbf{Closed-Source Prompting}} \\
GPT-4o$^\dagger$ & $75.3$ & $60.8$ & $31.2$ & $56.7$ & $21.6$ & $49.8$ & $48.0$ & $31.8$ & $23.7$ \\
Gemini-2.5-Pro$^\dagger$ & $92.8$ & $63.3$ & $62.1$ & $69.0$ & $26.6$ & $58.7$ & $60.3$ & $42.5$ & $35.9$ \\
\midrule
\multicolumn{10}{@{}l}{\textbf{Open-Source Prompting}} \\
Qwen2.5$^\dagger$ & $5.9$ & $5.5$ & $3.3$ & $9.7$ & $4.2$ & $0.0$ & $4.1$ & $23.1$ & $5.2$ \\
ReAct$^\dagger$ & $17.4$ & $20.5$ & $15.7$ & $6.2$ & $7.7$ & $2.0$ & $12.8$ & $40.1$ & $11.3$ \\
Reflexion$^\dagger$ & $35.3$ & $22.2$ & $21.7$ & $13.6$ & $19.4$ & $3.7$ & $21.8$ & $55.8$ & $21.9$ \\
\midrule
\multicolumn{10}{@{}l}{\textbf{RL-Based Training}} \\
PPO$^\dagger$ & $64.8$ & $40.5$ & $57.1$ & $60.6$ & $46.4$ & $47.4$ & $54.4$ & $73.8$ & $51.5$ \\
& $(\pm 3.5)$ & $(\pm 6.9)$ & $(\pm 4.9)$ & $(\pm 6.6)$ & $(\pm 4.0)$ & $(\pm 1.9)$ & $(\pm 3.1)$ & $(\pm 3.0)$ & $(\pm 2.9)$ \\
RLOO$^\dagger$ & $88.3$ & $52.8$ & $71.0$ & $62.8$ & $66.4$ & $56.9$ & $69.7$ & $73.9$ & $52.1$ \\
& $(\pm 3.0)$ & $(\pm 8.6)$ & $(\pm 5.9)$ & $(\pm 8.7)$ & $(\pm 5.5)$ & $(\pm 4.7)$ & $(\pm 2.5)$ & $(\pm 5.6)$ & $(\pm 6.7)$ \\
GRPO$^\dagger$ & $82.89$ & $82.14$ & $73.86$ & $78.57$ & $77.78$ & $71.43$ & $77.86$ & $84.73$ & $71.35$ \\
& $(\pm 3.6)$ & $(\pm 6.4)$ & $(\pm 6.8)$ & $(\pm 0.0)$ & $(\pm 4.5)$ & $(\pm 3.9)$ & $(\pm 1.3)$ & $(\pm 0.5)$ & $(\pm 2.1)$ \\
GiGPO$^\dagger$ & $98.81$ & $95.16$ & $81.46$ & $78.57$ & $94.44$ & $93.65$ & $90.88$ & $87.94$ & $73.83$ \\
& $(\pm 1.7)$ & $(\pm 3.9)$ & $(\pm 0.6)$ & $(\pm 0.0)$ & $(\pm 0.0)$ & $(\pm 5.9)$ & $(\pm 1.0)$ & $(\pm 0.4)$ & $(\pm 2.3)$ \\
GraphGPO$^\dagger$ & $95.15$ & $100.0$ & $85.26$ & $85.71$ & $96.30$ & $93.65$ & $92.71$ & $89.29$ & $78.65$ \\
& $(\pm 1.6)$ & $(\pm 0.0)$ & $(\pm 2.6)$ & $(\pm 5.8)$ & $(\pm 2.6)$ & $(\pm 2.2)$ & $(\pm 1.3)$ & $(\pm 1.5)$ & $(\pm 3.9)$ \\
\midrule
GiGPO$^\ast$ & $78.52$ & $87.26$ & $98.31$ & $96.24$ & $\mathbf{92.97}$ & $92.15$ & $90.17$ & $89.81$ & $76.17$ \\
& $(\pm 3.9)$ & $(\pm 0.7)$ & $(\pm 0.7)$ & $(\pm 1.9)$ & $\mathbf{(\pm 1.8)}$ & $(\pm 0.4)$ & $(\pm 0.1)$ & $(\pm 0.9)$ & $(\pm 1.6)$ \\
GraphGPO$^\ast$ & $78.96$ & $91.82$ & $\mathbf{98.61}$ & $99.44$ & $89.31$ & $94.19$ & $91.47$ & $88.24$ & $74.80$ \\
& $(\pm 3.8)$ & $(\pm 3.3)$ & $\mathbf{(\pm 0.3)}$ & $(\pm 0.8)$ & $(\pm 1.5)$ & $(\pm 2.0)$ & $(\pm 0.5)$ & $(\pm 1.4)$ & $(\pm 2.4)$ \\
\rowcolor{gray!20}
\textbf{MileGPO} & $\mathbf{90.47}$ & $\mathbf{93.29}$ & $98.31$ & $\mathbf{100.00}$ & $90.66$ & $\mathbf{98.02}$ & $\mathbf{94.60}$ & $\mathbf{90.29}$ & $\mathbf{78.58}$ \\
\rowcolor{gray!20}
& $\mathbf{(\pm 2.8)}$ & $\mathbf{(\pm 1.3)}$ & $(\pm 0.7)$ & $\mathbf{(\pm 0.0)}$ & $(\pm 1.5)$ & $\mathbf{(\pm 0.6)}$ & $\mathbf{(\pm 0.3)}$ & $\mathbf{(\pm 0.8)}$ & $\mathbf{(\pm 1.2)}$ \\
\bottomrule
\end{tabular}%
}
\endgroup
\caption{Performance (\%) on ALFWorld and WebShop, grouped into closed-source prompting, open-source prompting, and RL-based training. Rows marked with $\dagger$ are reported by GraphGPO \cite{cheng2026trajectorylevelattributiongraphbasedcredit}, rows marked with $\ast$ are our reproduced baselines, and the unmarked row denotes MileGPO.}
% \caption{Performance (\%) on ALFWorld and WebShop. $\dagger$: GraphGPO-reported; $\ast$: reproduced; no mark: ours. All results are averaged over three random seeds.}
\label{tab:main-results}
\end{table*}

\subsection{MileGPO Return and Policy Update}

% [OLD] MileGPO inserts the PCC scores into the positive RCS potential:
MileGPO instantiates the $\Phi^{+}$ of Equations~\ref{eq:positive-shaping} and~\ref{eq:negative-shaping} by inserting the PCC scores into the RCS form:
\begin{equation}
 \Phi^{+}(s)=\max_{v\in\mathcal{S}^{+}}
 \overline S_{\mathrm{pcc}}^{+}(v)\,\omega^{d(s,v)}.
 \label{eq:positive-potential}
\end{equation}
The trap potential remains $\Phi^{-}=\Phi_{\mathrm{R}}^{-}$ from Equation~\ref{eq:reliability-calibrated-potential}. Using the resulting increments $\delta_t^{\pm}$, and suppressing the trajectory index on step-level quantities for readability, the MileGPO step return, denoted by $r_t^{\mathrm{M}}$, is
\begin{equation}
 r_t^{\mathrm{M}}=
 c\,\gamma_G^{d(s_{t+1},g)}
 +c\lambda\left(\delta_t^{+}-\delta_t^{-}\right).
 \label{eq:milegpo-return}
\end{equation}
The first term is the GraphGPO return, and $\lambda$ scales the milestone and trap correction. Directly normalizing their sum can let a small correction dominate when GraphGPO returns are tied. We instead normalize the graph and MileGPO returns separately. Let $\mathrm{Norm}_{q,s_t}$ normalize transitions with the same task and source state, and let $r_t^{\mathrm{G}}=r^{\mathrm{G}}(s_t,s_{t+1})$:
% The first term is the GraphGPO return and $\lambda$ scales the correction. We normalize the graph and MileGPO returns separately to prevent ties from being dominated; let $\mathrm{Norm}_{q,s_t}$ normalize transitions with the same task and source state, and let $r_t^{\mathrm{G}}=r^{\mathrm{G}}(s_t,s_{t+1})$:
\begin{equation}
 \widehat A_t^{\mathrm{G}}
 =\mathrm{Norm}_{q,s_t}(r_t^{\mathrm{G}}),\quad
 \widehat A_t^{\mathrm{mix}}
 =\mathrm{Norm}_{q,s_t}(r_t^{\mathrm{M}}).
 \label{eq:separated-step-advantages}
\end{equation}
% [OLD] Their difference, $\widehat A_t^{\mathrm{res}}=\widehat A_t^{\mathrm{mix}}- \widehat A_t^{\mathrm{G}}$, isolates the milestone correction. With trajectory score $Z_i=\sum_{\ell}r_{i,\ell}^{\mathrm{tok}}$, each token in action $a_t$ receives
% [OLD-v2] Their difference, $\widehat A_t^{\mathrm{res}}=\widehat A_t^{\mathrm{mix}}- \widehat A_t^{\mathrm{G}}$, isolates the milestone correction. Let $Z_i=\sum_{\ell}r_{i,\ell}^{\mathrm{tok}}$ be the trajectory score of $\tau_i$, summing its token-level rewards $r_{i,\ell}^{\mathrm{tok}}$, and let $\mathrm{Norm}_{q}$ normalize these scores within the rollout group of task $q$. Each token in action $a_t$ then receives:
% [OLD-v3] Here, $\widehat A_t^{\mathrm{G}}$ is the step advantage obtained from the GraphGPO return alone, while $\widehat A_t^{\mathrm{mix}}$ is the step advantage obtained from the shaped MileGPO return. Their difference, $\widehat A_t^{\mathrm{res}}=\widehat A_t^{\mathrm{mix}}- \widehat A_t^{\mathrm{G}}$, isolates the milestone correction. Let $Z_i=\sum_{\ell}r_{i,\ell}^{\mathrm{tok}}$ be the trajectory score of $\tau_i$, summing its token-level rewards $r_{i,\ell}^{\mathrm{tok}}$, and let $\mathrm{Norm}_{q}$ normalize these scores within the rollout group of task $q$. Each token in action $a_t$ then receives:
Here, $\widehat A_t^{\mathrm{G}}$ is the step advantage obtained from the GraphGPO return alone, while $\widehat A_t^{\mathrm{mix}}$ is the step advantage obtained from the shaped MileGPO return. Their difference, $\widehat A_t^{\mathrm{res}}=\widehat A_t^{\mathrm{mix}}- \widehat A_t^{\mathrm{G}}$, isolates the milestone correction. The episode-level term restores the trajectory index: let $Z_i=\sum_{\ell}r_{i,\ell}^{\mathrm{tok}}$ be the trajectory score of $\tau_i$, summing the per-token rewards $r_{i,\ell}^{\mathrm{tok}}$ that the training framework assigns to its $\ell$-th token, and let $\mathrm{Norm}_{q}$ normalize these scores within the rollout group of task $q$. Each token in action $a_t$ then receives:
\begin{equation}
A_t=w_{\mathrm{step}}\left[
\widehat A_t^{\mathrm{G}}
+\eta\widehat A_t^{\mathrm{res}}\right]
+w_{\mathrm{episode}}\,\mathrm{Norm}_{q}(Z_i).
\label{eq:joint-advantage}
\end{equation}
% [OLD] Here, $w_{\mathrm{step}}$ and $w_{\mathrm{episode}}$ weight step- and episode-level advantages. The coefficient $\eta$ controls the correction: $\eta=0$ recovers GraphGPO, $\eta=1$ recovers the normalized MileGPO return, and intermediate values interpolate between them.
Here, $w_{\mathrm{step}}$ and $w_{\mathrm{episode}}$ weight step- and episode-level advantages, and $\eta$ scales the milestone correction.

% [OLD] The actor uses the same clipped token-level objective and KL regularization as the group-based baselines. \textbf{MD, RCS, and PCC use only the rollout graph and the success or failure of trajectories in the current on-policy batch.} MileGPO therefore requires no external milestone annotations, critic, verifier, reward model, extra environment interaction, or auxiliary model inference.
The actor uses the same clipped token-level objective and KL regularization as the group-based baselines. \textbf{MD, RCS, and PCC use only the rollout graph and on-policy outcome rewards}, requiring no external process annotations, critic, reward model, or auxiliary inference.

\section{Experiments}
\label{sec:experiments}

% \begin{table}[t]
% \centering
% \begingroup
% \small
% \setlength{\tabcolsep}{3.5pt}
% \newcommand{\ids}[2]{$#1_{\mbox{\scriptsize$\pm #2$}}$}
% \newcommand{\idsb}[2]{$\mathbf{#1}_{\mbox{\scriptsize$\pm\mathbf{#2}$}}$}
% \begin{tabular}{@{}lccc@{}}
% \toprule
% \textbf{Method} & \textbf{ID $\uparrow$} & \textbf{OOD $\uparrow$} & \textbf{Gap $\downarrow$} \\
% \midrule
% GiGPO & \ids{92.06}{0.72} & \ids{90.17}{0.09} & $1.89$ \\
% GraphGPO & \ids{95.25}{0.60} & \ids{91.47}{0.51} & $3.78$ \\
% \textbf{MileGPO} & \idsb{96.29}{0.55} & \idsb{94.60}{0.33} & $\mathbf{1.69}$ \\
% \bottomrule
% \end{tabular}
% \caption{ALFWorld success rates (\%) on the ID and OOD test sets, averaged over three evaluation seeds. Gap is ID minus OOD; lower is better.}
% \label{tab:alfworld-id-ood}
% \endgroup
% \end{table}

\begin{table}[t]
\centering
\begingroup
% AAAI 模板允许将列间距压缩至 1 mm；括号内标准差使用 9 pt。
% \setlength{\tabcolsep}{1mm}
\newcommand{\meanstd}[2]{$#1\,\mbox{\small$(#2)$}$}
\newcommand{\meanstdb}[2]{$\mathbf{#1}\,\mbox{\small$\mathbf{(#2)}$}$}
\resizebox{\columnwidth}{!}{%
\begin{tabular}{@{}lccc@{}}
\toprule
\textbf{Method} & \textbf{ID $\uparrow$} & \textbf{OOD $\uparrow$} & \textbf{Gap $\downarrow$} \\
\midrule
GiGPO & \meanstd{92.06}{0.72} & \meanstd{90.17}{0.09} & $1.89$ \\
GraphGPO & \meanstd{95.25}{0.60} & \meanstd{91.47}{0.51} & $3.78$ \\
\textbf{MileGPO} & \meanstdb{96.29}{0.55} & \meanstdb{94.60}{0.33} & $\mathbf{1.69}$ \\
\bottomrule
\end{tabular}%
}
\caption{ALFWorld success rates (\%) on the ID and OOD.}
\label{tab:alfworld-id-ood}
\endgroup
\end{table}

\begin{table}[t]
\centering
\begingroup
\small
% 单栏消融表使用模板允许的 9 pt 字号和最小列间距。
\setlength{\tabcolsep}{1mm}
\newcommand{\ablrs}[2]{$#1\,\mbox{\small$(#2)$}$}
\newcommand{\ablrsb}[2]{$\mathbf{#1}\,\mbox{\small$\mathbf{(#2)}$}$}
\resizebox{\columnwidth}{!}{%
\begin{tabular}{@{}lcccc@{}}
\toprule
& \multicolumn{2}{c}{\textbf{ALFWorld}} & \multicolumn{2}{c}{\textbf{WebShop}} \\
\cmidrule(lr){2-3}\cmidrule(lr){4-5}
\textbf{Method} & \textbf{ID $\uparrow$} & \textbf{OOD $\uparrow$} & \textbf{Score $\uparrow$} & \textbf{Succ. $\uparrow$} \\
\midrule
\textbf{MileGPO} & \ablrsb{96.3}{0.6} & \ablrsb{94.6}{0.3} & \ablrsb{90.3}{0.8} & \ablrsb{78.6}{1.2} \\
w/o PCC & \ablrs{95.9}{0.5} & \ablrs{90.9}{0.5} & \ablrs{87.9}{1.9} & \ablrs{77.2}{2.2} \\
w/o PCC, RCS & \ablrs{95.3}{0.0} & \ablrs{89.7}{0.3} & \ablrs{87.4}{1.5} & \ablrs{73.8}{2.2} \\
\bottomrule
\end{tabular}%
}
\endgroup
\caption{MileGPO ablation on ALFWorld and WebShop (\%).}
\label{tab:cumulative-ablation}
\end{table}

% We evaluate MileGPO together with two controlled predecessor variants. \textbf{Milestone Discovery} tests rollout-native intermediate milestones under uniform propagation. \textbf{Reliability-Calibrated Shaping (RCS)} tests a unified reliability-weighted and polarity-consistent shaping rule. The full \textbf{MileGPO} method further applies PCC with progress and branch-counterfactual evidence to downweight weak anchors. The experiments examine four aspects of the resulting credit signal: whether MileGPO improves final task performance, whether its gains align with the structural ambiguity of the rollout graph, how the three nested designs resolve different sources of noisy credit, and whether this refinement adds meaningful computational cost.

\subsection{Experimental Setup}

\subsubsection{Benchmarks and Evaluation}
\mbox{ALFWorld}~\cite{shridhar2021alfworldaligningtextembodied} requires an agent to complete household tasks through text-based navigation and object manipulation, while WebShop~\cite{yao2023webshopscalablerealworldweb} requires it to search for and purchase products that satisfy natural-language instructions. They evaluate embodied planning and grounded web interaction, respectively. For ALFWorld, we train on 3,553 tasks and evaluate on both the in-distribution (ID) and out-of-distribution (OOD) splits. For WebShop, we follow the repository's evaluation protocol. GiGPO and GraphGPO serve as the primary step-group and graph-credit baselines, respectively. We report success rates on both benchmarks and WebShop task scores. For ALFWorld, we additionally include the ID--OOD gap to evaluate generalization. Most results are averaged over 3 random seeds during testing.

\subsubsection{Implementation Details}
All local runs use Qwen2.5-1.5B-Instruct~\cite{qwen2024qwen25} and share the model, rollout, optimizer, and environment configuration. Following the shared agent protocol~\cite{feng2025groupingrouppolicyoptimizationllm}, the policy observes the two most recent interaction steps and generates reasoning within \texttt{<think>} tags followed by an action within \texttt{<action>} tags. Each iteration samples 16 task groups with eight rollouts per group. We train on two NVIDIA H20 GPUs. Since the validation curves of the original GraphGPO and GiGPO implementations under their reported configurations had not converged, we extended training for all local methods to 300 optimization steps and selected the checkpoint with the best validation performance for evaluation. We use maximum horizons of 50 steps on ALFWorld and 30 on WebShop, an actor learning rate of $10^{-6}$, a KL coefficient of $0.01$, and an invalid-action penalty of $0.1$. More implementation details are provided in Appendix A.

\subsection{Main Results}

Table~\ref{tab:main-results} shows that MileGPO consistently improves over the reproduced baselines. On ALFWorld, it raises overall success by $3.13$ points over GraphGPO and $4.43$ points over GiGPO. On WebShop, it improves success by $3.78$ and $2.41$ points and task score by $2.05$ and $0.48$ points over GraphGPO and GiGPO, respectively. The gains in both strict success and graded task score indicate better task completion as well as higher-quality partial progress.

Table~\ref{tab:alfworld-id-ood} further shows that MileGPO reduces the ID--OOD gap to $1.69$ points, compared with $3.78$ for GraphGPO and $1.89$ for GiGPO. This stronger generalization is consistent with MileGPO learning transferable local decision preferences: RCS suppresses milestones with weak or mixed outcome support, while PCC distinguishes competing transitions through local progress and BCC.

MileGPO also exhibits lower result variance than the reproduced GraphGPO baseline on both benchmarks. By resolving local credit ties rather than relying only on final-goal distance, MileGPO provides a more discriminative credit signal, offering a plausible explanation for its simultaneous improvements in performance, generalization, and stability.

\subsection{Ablation Study}

Table~\ref{tab:cumulative-ablation} evaluates PCC and RCS through cumulative removal. Each removal degrades performance across both benchmarks, while the relatively small changes on ALFWorld ID contrast with the larger losses on ALFWorld OOD and WebShop, indicating that the two components primarily improve generalization and task completion rather than fitting the in-distribution evaluation set. Removing PCC decreases ALFWorld OOD success by $3.71$ points and widens the ID--OOD gap by $3.32$ points, while reducing WebShop task score and success by $2.38$ and $1.43$ points, suggesting that local progress and branch comparisons primarily aid generalization and partial task completion. Further removing RCS causes the largest additional drop in WebShop success ($3.32$ points) and a $1.24$-point decrease on ALFWorld OOD. PCC and RCS therefore play complementary roles: PCC sharpens local preferences and supports transfer, whereas RCS suppresses weak or inconsistent milestones and stabilizes which intermediate signals are propagated.

\begin{figure*}[t]
\centering
\includegraphics[width=\textwidth]{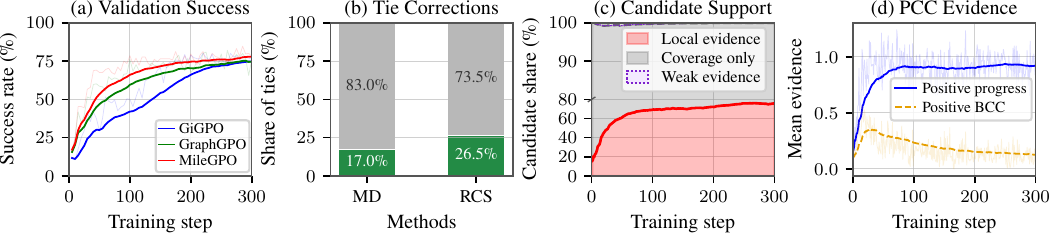}
\caption{WebShop credit-calibration diagnostics. (a) Validation success compares three locally trained methods. (b) Tie-correction rate measures the fraction of opposite-outcome transition pairs with equal GraphGPO credit, where shaped returns rank successful transitions higher. (c) Candidate support partitions success-visited candidates into those with local evidence (positive progress or BCC), successful-path coverage only, and weak candidates. (d) PCC evidence separates positive progress from positive sibling-branch contrast.}
\label{fig:pcc-replay-composite}
\end{figure*}

\subsection{Credit-Calibration Diagnostics}

To further examine how MileGPO affects the training process, we use WebShop as a representative case because it exhibits a higher rate of final-distance ties than ALFWorld. Figure~\ref{fig:pcc-replay-composite} tracks the evolution of candidate evidence and validation performance during training, as well as whether shaped returns recover preference information obscured by GraphGPO's final-distance credit. Panel (a) reports validation success throughout optimization, showing the training dynamics of the three methods. Panel (b) shows that RCS corrects a larger fraction of GraphGPO ties than uniform Milestone Discovery, indicating that reliability-weighted shaping recovers more outcome-consistent preferences. Panel (c) partitions success-visited candidates before PCC reweighting. The increasing local-evidence share suggests that more candidate milestones are supported by transition-level evidence. More illustrative examples of how MileGPO resolves ambiguous credit are provided in Appendix B.

\subsection{Mechanism Analysis Across Environments}

Figure~\ref{fig:motivation-metrics} and Table~\ref{tab:main-results} show that MileGPO benefits from resolving such ambiguity. ALFWorld and WebShop have nearly identical shared-state transition coverage ($73.7\%$ versus $73.9\%$), showing that both environments contain enough recurring structure to mine intermediate anchors. The decisive difference is ambiguity within that structure: final-goal distance ties $54.4\%$ of same-state action pairs in ALFWorld but $72.7\%$ in WebShop. Accordingly, MileGPO produces its clearest gain on WebShop, improving both success and task score. This result distinguishes the source of the gain from graph reuse alone. GraphGPO already aggregates recurring states, yet a distance-based return cannot rank actions that lead to equally distant successors. MileGPO extracts the missing signal from two relationships available in the same rollout group: whether a transition makes local progress and whether it outperforms sibling transitions from the same parent. The larger WebShop improvement, therefore, matches the mechanism precisely: the environment with more unresolved local comparisons benefits more from PCC.

\subsection{Efficiency and Computational Cost}
\begin{figure}[t]
    \centering
    \includegraphics[width=0.98\columnwidth]{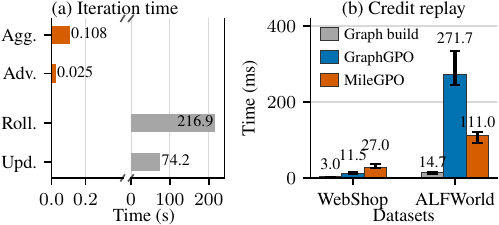}
    \caption{Efficiency analysis. (a) Published GraphGPO iteration times;
    Agg., Adv., Roll., and Upd. denote graph aggregation, graph advantage,
     rollout, and policy update, respectively.
    (b) Median single-thread CPU replay time with IQRs over five stored batches;
    graph construction is shared.}
    \label{fig:efficiency-overview}
\end{figure}

We reuse the rollout graph and add no environment interactions, model passes, critics, parameters, or activation caches. For a graph with $|\mathcal{V}|$ nodes, $|\mathcal{E}|$ edges, $N$ transitions, and maximum out-degree $\Delta$, final-distance search costs $O((|\mathcal{V}|+|\mathcal{E}|)\log|\mathcal{V}|)$, candidate statistics cost $O(N)$, and sibling comparisons cost $O(|\mathcal{E}|\Delta)$. The two target potentials add $O(|\mathcal{V}|+|\mathcal{E}|)$ storage, leaving model-scale and rollout complexity unchanged. Figure~\ref{fig:efficiency-overview}(a) compares post-rollout graph processing with the substantially larger rollout and policy-update costs, while Figure~\ref{fig:efficiency-overview}(b) reports single-thread CPU replay results on WebShop and ALFWorld. Across both analyses, MileGPO's post-rollout computation remains lightweight relative to model-level training costs. These results show that MileGPO improves the existing credit signal without additional environment interactions or model inference, while confining its overhead to graph processing.

\section{Conclusion}

% [OLD] We study a central limitation of graph-based credit assignment for long-horizon LLM agents: final-goal distance captures reachability. Still, it can leave locally competing actions indistinguishable and cannot determine whether a success-visited state represents reliable progress.
% [OLD] We introduce \textbf{MileGPO}, which extracts intermediate credit from grouped on-policy rollouts through milestone and trap discovery, reliability-calibrated shaping, and progress-contrastive calibration with same-state branch evidence.
% [OLD] Because MileGPO only post-processes the rollout graph, it introduces no additional supervision, auxiliary models, or environment interaction.
% [OLD] Across ALFWorld and WebShop, MileGPO achieves state-of-the-art performance.
% [OLD] Ablations show that milestone discovery alone is insufficient, whereas reliability weighting and local progress and branch evidence provide complementary improvements.
% [OLD] Credit-calibration diagnostics and replay profiling further show that these gains align with resolving local credit ambiguity while confining the added computation to lightweight post-rollout graph processing.
% [OLD] These results demonstrate that existing rollout groups contain sufficient local evidence to provide useful process-level credit without additional supervision or interaction.
We study a central limitation of graph-based credit assignment for long-horizon LLM agents: final-goal distance captures reachability but can leave locally competing actions indistinguishable and cannot determine whether a success-visited state represents reliable progress. We introduce \textbf{MileGPO}, which extracts intermediate credit from grouped on-policy rollouts through milestone and trap discovery, reliability-calibrated shaping, and progress-contrastive calibration with same-state branch evidence, requiring no additional supervision, auxiliary models, or environment interaction. Across ALFWorld and WebShop, MileGPO achieves state-of-the-art performance. Ablations show that milestone discovery alone is insufficient, whereas reliability weighting and local progress and branch evidence provide complementary improvements. Credit-calibration diagnostics and replay profiling further show that these gains align with resolving local credit ambiguity while confining the added computation to lightweight post-rollout graph processing. These results demonstrate that existing rollout groups contain sufficient local evidence to provide useful process-level credit.

% \section*{Limitations}

% MileGPO derives its evidence from grouped on-policy rollouts, so its credit
% quality depends on sufficient state and transition recurrence and on diverse
% outcomes within each group. Sparse rollout graphs or batches with few
% successful trajectories may yield weak milestone and trap statistics. Our
% evaluation is limited to ALFWorld and WebShop with Qwen2.5-1.5B-Instruct; the
% results do not establish generalization to other environments, larger models,
% or multimodal agents. In addition, same-state comparison relies on the graph's
% state representation identifying equivalent situations: noisy or overly
% coarse state matching may miss useful comparisons or merge distinct states.
% Finally, the theoretical properties concern the credit transformation on a
% fixed sampled rollout graph and do not by themselves guarantee policy
% improvement across optimization updates.

\bibliography{milegpo}

% ARR permits appendices in the same PDF after the references.
% \clearpage
\appendix

\section{Implementation and Reproducibility Details}
\label{app:implementation-details}

\subsection{Training Procedure}
\label{app:training-procedure}

Algorithm~\ref{alg:milegpo-training} summarizes how MileGPO is inserted into
the grouped-rollout training pipeline. The rollout and graph-construction
stages are shared with GraphGPO; MileGPO changes the computation between graph
construction and the policy update. In particular, MD, RCS, and PCC are all
computed from the current task-local rollout graph and its observed outcomes.
They do not invoke the environment or an auxiliary model.

The separate normalization in line 10 is important when several outgoing
transitions are tied under final-goal distance: it prevents a numerically small
potential correction from replacing the scale of the graph return. The policy
update instead adds the normalized residual between the shaped and graph-only
advantages, as defined in the main method section.

\subsection{Shared Training and Evaluation Configuration}

All locally evaluated methods use Qwen2.5-1.5B-Instruct as the policy. Each
training iteration samples 16 task groups with eight rollouts per group. The
agent retains the two most recent interaction steps and is limited to 50
environment steps on ALFWorld and 30 on WebShop. We train for 300 optimization
steps with an actor learning rate of $10^{-6}$, a KL coefficient of $0.01$,
and an invalid-action penalty of $0.1$. Validation uses stochastic decoding at
temperature $0.4$.

Shared model, rollout, optimizer, and environment settings are held fixed
across the local comparisons. Method-specific return discounts follow the
corresponding baseline and MileGPO recipes. Advantage normalization is a
recipe-level setting: ALFWorld uses mean--standard-deviation normalization,
whereas WebShop uses mean-only normalization. For each reported method, we run
three test evaluations with seeds 123, 456, and 789, and report their mean and
population standard deviation. WebShop evaluation uses 512 examples per seed
from the official-small split; ALFWorld evaluation uses the full out-of-distribution split of 134 tasks.

Table~\ref{tab:hyperparams-shared} lists settings shared by both environments.
The environment-specific settings are 128/256 training-time validation
examples, 50/30 maximum environment steps, mean--standard-deviation versus
mean-only advantage normalization, 256/64 PPO mini-batch sizes, and 64/4 actor
and rollout/reference micro-batch sizes per GPU for ALFWorld/WebShop,
respectively.

\begin{table*}[!t]
\centering
\small
\setlength{\tabcolsep}{1mm}
\resizebox{\textwidth}{!}{%
\begin{tabular}{@{}llcc@{}}
\toprule
\textbf{Group} & \textbf{Hyperparameter} & \textbf{ALFWorld} & \textbf{WebShop} \\
\midrule
\emph{Model and data}
& Policy model & Qwen2.5-1.5B-Instruct & Qwen2.5-1.5B-Instruct \\
& Training task groups / iteration & 16 & 16 \\
& Rollouts per task group ($K$) & 8 & 8 \\
& Maximum prompt / response tokens & 4096 / 512 & 4096 / 512 \\
& History length & 2 & 2 \\
\addlinespace
\emph{Optimization}
& Actor learning rate & $1\times10^{-6}$ & $1\times10^{-6}$ \\
& KL loss coefficient / type & 0.01 / low-variance KL & 0.01 / low-variance KL \\
& Invalid-action penalty coefficient & 0.1 & 0.1 \\
& Optimization steps / validation interval & 300 / 5 & 300 / 5 \\
\addlinespace
\emph{Rollout and evaluation}
& Tensor model parallel size & 2 & 2 \\
& vLLM GPU memory utilization & 0.40 & 0.40 \\
& Validation decoding & temperature 0.4, sampling enabled & temperature 0.4, sampling enabled \\
& Training seed & 0 & 0 \\
& Test seeds & 123, 456, 789 & 123, 456, 789 \\
& Checkpoint selection & best validation success rate & best validation success rate \\
\bottomrule
\end{tabular}%
}
\caption{Training and evaluation hyperparameters shared by the reported local
Qwen2.5-1.5B-Instruct runs on ALFWorld and WebShop. The two environment columns are
shown separately for direct comparison.}
\label{tab:hyperparams-shared}
\end{table*}

\subsection{Baseline Alignment and Checkpoint Selection}
\label{app:baseline-alignment}

\begin{algorithm}[t]
\caption{MileGPO training procedure}
\label{alg:milegpo-training}
\begin{algorithmic}[1]
\REQUIRE Policy $\pi_{\theta}$, task distribution $p(q)$, group size $K$,
maximum horizon $T$, graph discount $\gamma_G$, shaping coefficient $\lambda$,
residual correction coefficient $\eta$
\FOR{each training iteration}
    \STATE Set the rollout policy to the current policy
    \STATE Sample task groups $q\sim p(q)$ and collect $K$ trajectories per task
    \STATE Canonicalize observations and construct one directed rollout graph $G_q$ per task
    \STATE Compute final-goal distances and GraphGPO returns $r^{\mathrm G}$
    \STATE Mine candidate milestones and traps from visitation and outcome statistics (MD)
    \STATE Weight positive and negative candidates and construct the two target potentials (RCS)
    \STATE Compute local-progress and same-parent branch evidence; retain, amplify, or shrink milestone candidates (PCC)
    \STATE Form the shaped step returns $r^{\mathrm M}$ from $r^{\mathrm G}$ and the potential increments
    \STATE Normalize $r^{\mathrm G}$ and $r^{\mathrm M}$ within each task--source-state group
    \STATE Combine the graph, residual milestone, and episode-level advantages
    \STATE Update $\theta$ with the clipped policy objective and KL regularization
\ENDFOR
\end{algorithmic}
\end{algorithm}

The locally reproduced GiGPO and GraphGPO baselines use the same policy model,
task data, prompt construction, rollout group size, interaction horizon,
optimizer settings, validation interval, and hardware allocation as MileGPO.
The comparison therefore changes the credit estimator while preserving the
agent--environment interface and model-level training workload. We retain
method-specific return definitions and recipe-level normalization conventions
rather than forcing them into a common estimator. Results marked as reproduced
in the main table come from these aligned local runs; results explicitly marked
as reported are transcribed from the corresponding source paper and are not
treated as controlled local comparisons.

All aligned local methods are trained for 300 optimization steps because their
validation curves had not converged under the shorter reported schedule. We
evaluate the checkpoint with the highest training-time validation success rate. ALFWorld ID and
OOD use separate evaluation splits, while WebShop uses the official-small
protocol.

\subsection{Agent Prompt Templates}
\label{app:prompt-templates}

The prompts are shared by MileGPO and the locally reproduced baselines. At each
interaction step, placeholders are filled with the task, the two most recent
observation--action pairs, the current observation, and the admissible actions.
The reasoning and executable action are enclosed by \texttt{<think>} and
\texttt{<action>} tags, respectively.

\begin{figure*}[!t]
\centering
\includegraphics[width=\textwidth]{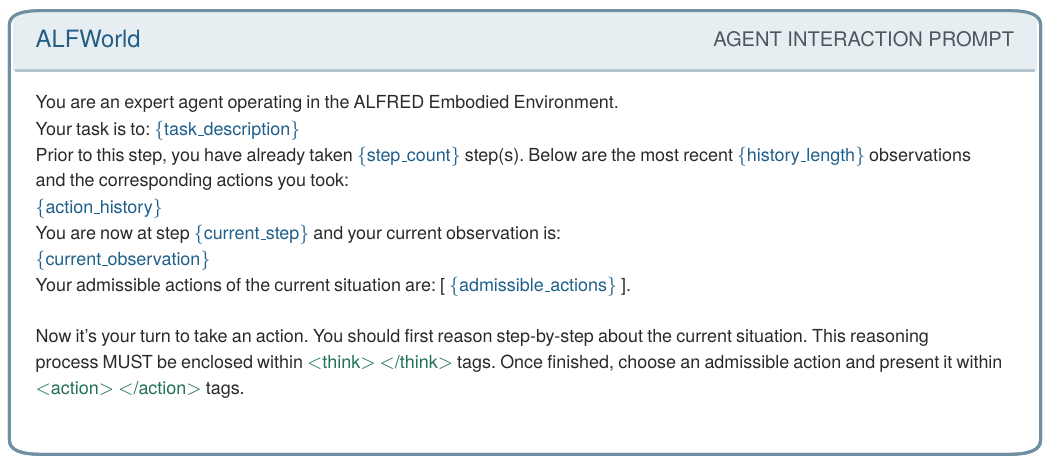}
\caption{ALFWorld prompt template. Runtime placeholders are blue and output
tags are green.}
\label{fig:alfworld-prompt}
\end{figure*}

\begin{figure*}[!t]
\centering
\includegraphics[width=\textwidth]{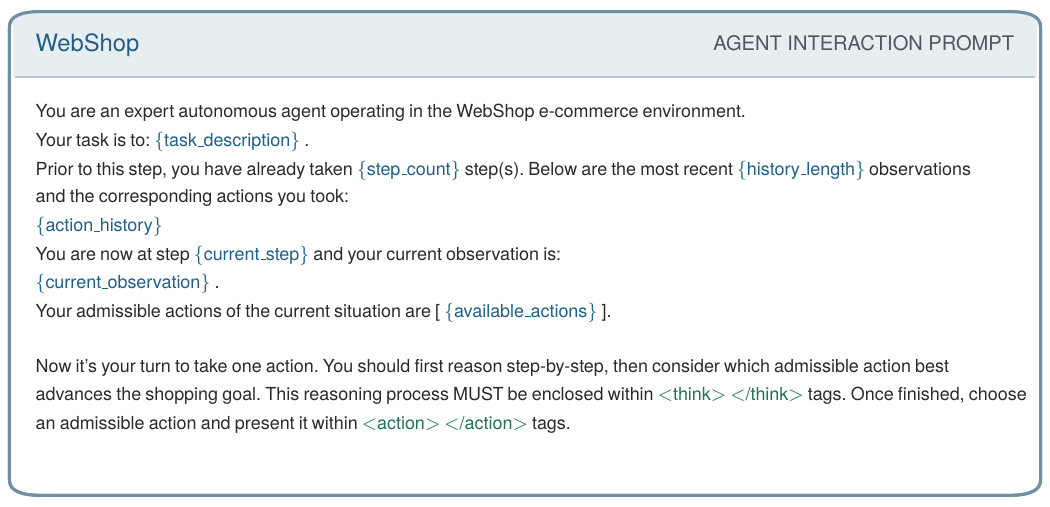}
\caption{WebShop prompt template. Runtime placeholders are blue and output
tags are green.}
\label{fig:webshop-prompt}
\end{figure*}

\subsection{MileGPO Hyperparameters}

For MileGPO, we use $c=10$, $\omega=0.20$, $\gamma_{\Phi}=1$, and
$\lambda=0.25$, with $\gamma_G=0.20$ in both environments. The milestone and
trap weights are $w_{+}=1$ and $w_{-}=0.25$. Candidate mining uses
$w_s=w_m=w_f=1$, $w_c=w_l=0.25$, a minimum node support of one trajectory, and
a trap threshold of $0.1$ with at least two failed-trajectory visits unless a
revisit is observed. The PCC stage sets
$\alpha_d=\alpha_s=\alpha_f=w_{\mathrm{bc}}=w_{\mathrm{pg}}=w_{\mathrm{pcc}}=1$
in both environments. The policy update uses
$w_{\mathrm{step}}=w_{\mathrm{episode}}=1$. ALFWorld uses the selected strict
configuration with $\theta_m=1.1$, $\rho=0$, $\eta=0.20$, and progress-only
candidate eligibility
($\kappa_{\mathrm{bc}}=0$).
Because $m(v)\leq1$, this configuration requires positive progress evidence to
retain a milestone; BCC calibrates its strength but cannot establish eligibility
by itself. The WebShop run reported in the main table uses $\theta_m=0.5$,
$\rho=0.5$, and $\eta=1$, together with
the full progress-or-branch eligibility rule ($\kappa_{\mathrm{bc}}=1$) for its more
ambiguous shopping-state graph. These are environment-level recipe parameters;
the MileGPO mechanism itself is unchanged.

Table~\ref{tab:hyperparams-milegpo} gives the MileGPO configuration used in
the local comparisons. Symbols follow the definitions in the method section.

\begin{table*}[!t]
\centering
\small
\setlength{\tabcolsep}{1mm}
\resizebox{\textwidth}{!}{%
\begin{tabular}{@{}lllcc@{}}
\toprule
\textbf{Stage} & \textbf{Parameter} & \textbf{Role} & \textbf{ALFWorld} & \textbf{WebShop} \\
\midrule
\emph{Graph return}
& $c$ & Return scale & 10 & 10 \\
& $\gamma_G$ & Graph-distance discount & 0.20 & 0.20 \\
& $\gamma_{\Phi}$ & Potential discount & 1.00 & 1.00 \\
& $\lambda$ & Shaping coefficient & 0.25 & 0.25 \\
\addlinespace
\emph{MD and RCS}
& $w_{+}, w_{-}$ & Positive / negative potential weights & 1.00 / 0.25 & 1.00 / 0.25 \\
& $w_s,w_m,w_f$ & Success, coverage, failure weights & 1 / 1 / 1 & 1 / 1 / 1 \\
& $w_c,w_l$ & Centrality and revisit weights & 0.25 / 0.25 & 0.25 / 0.25 \\
& Minimum node support & Distinct trajectory count & 1 & 1 \\
& Minimum trap visits & Failed-trajectory support & 2 & 2 \\
& Minimum trap score & Trap admission threshold & 0.10 & 0.10 \\
& $\omega$ & Milestone/trap propagation discount & 0.20 & 0.20 \\
& \texttt{return\_mode} & Graph and potential composition & \multicolumn{2}{c}{\texttt{graph\_plus\_potential}} \\
\addlinespace
\emph{PCC}
& $\alpha_d,\alpha_s,\alpha_f$ & Distance, success, failure terms & 1 / 1 / 1 & 1 / 1 / 1 \\
& $w_{\mathrm{bc}},w_{\mathrm{pg}}$ & BCC and progress evidence & 1 / 1 & 1 / 1 \\
& $w_{\mathrm{pcc}}$ & PCC reweighting strength & 1 & 1 \\
& $\theta_m$ & Coverage retention threshold & 1.1 & 0.5 \\
& $\kappa_{\mathrm{bc}}$ & Branch-only retention switch & 0 & 1 \\
& $\rho$ & Shrink factor for weak candidates & 0 & 0.5 \\
\addlinespace
\emph{Policy update}
& $w_{\mathrm{step}},w_{\mathrm{episode}}$ & Step / episode advantage weights & 1 / 1 & 1 / 1 \\
& $\eta$ & Residual milestone correction weight & 0.20 & 1.00 \\
\bottomrule
\end{tabular}%
}
\caption{MileGPO hyperparameters. MD denotes milestone discovery, RCS denotes
reliability-calibrated shaping, and PCC denotes progress-contrastive
calibration. The environment-specific PCC controls are fixed during training
and evaluation, as are the policy-update weights.}
\label{tab:hyperparams-milegpo}
\end{table*}

\section{Theoretical Properties of MileGPO}
\label{app:theoretical-properties}

This section establishes three properties of the MileGPO credit estimator on a
fixed on-policy rollout graph. The results concern the credit transformation
computed from one sampled batch; they do not require assumptions about future
rollouts or the environment dynamics. Let the step-level component of the token
advantage in Equation~\ref{eq:joint-advantage} be
\begin{equation}
\begin{aligned}
 B_t&=A_t-w_{\mathrm{episode}}\,\mathrm{Norm}_{q}(Z_i)\\
 &=w_{\mathrm{step}}\left[\widehat A_t^{\mathrm{G}}
 +\eta\widehat A_t^{\mathrm{res}}\right].
\end{aligned}
 \label{eq:proof-step-component}
\end{equation}
We assume $w_{\mathrm{step}}\geq0$ and $\eta\in[0,1]$, as in the reported
configuration. All normalized advantages below are computed within the same
task--source-state group.

\begin{proposition}[Controlled residual interpolation]
\label{prop:residual-interpolation}
Define the GraphGPO and fully shaped step components as
$B_t^{\mathrm{G}}=w_{\mathrm{step}}\widehat A_t^{\mathrm{G}}$ and
$B_t^{\mathrm{M}}=w_{\mathrm{step}}\widehat A_t^{\mathrm{mix}}$,
respectively. Then MileGPO satisfies
\begin{equation}
 B_t=(1-\eta)B_t^{\mathrm{G}}+\eta B_t^{\mathrm{M}},
 \label{eq:proof-convex-interpolation}
\end{equation}
and its deviation from the GraphGPO step component is
\begin{equation}
 |B_t-B_t^{\mathrm{G}}|
 =w_{\mathrm{step}}\eta
 |\widehat A_t^{\mathrm{mix}}-\widehat A_t^{\mathrm{G}}|.
 \label{eq:proof-deviation-bound}
\end{equation}
Consequently, $\eta=0$ recovers the GraphGPO step component, $\eta=1$
recovers the normalized shaped step component, and intermediate values provide
a convex interpolation between them.
\end{proposition}

\begin{proof}
By definition,
$\widehat A_t^{\mathrm{res}}=\widehat A_t^{\mathrm{mix}}
-\widehat A_t^{\mathrm{G}}$. Substituting this identity into
Equation~\ref{eq:proof-step-component} gives
\begin{equation}
\begin{array}{rcl}
B_t&=&w_{\mathrm{step}}\bigl[\widehat A_t^{\mathrm{G}}
+\eta(\widehat A_t^{\mathrm{mix}}-\widehat A_t^{\mathrm{G}})\bigr]\\
&=&(1-\eta)B_t^{\mathrm{G}}+\eta B_t^{\mathrm{M}}.
\end{array}
\end{equation}
Subtracting $B_t^{\mathrm{G}}$ and taking absolute values yields
Equation~\ref{eq:proof-deviation-bound}. Since $\eta\in[0,1]$, the two
coefficients in Equation~\ref{eq:proof-convex-interpolation} are nonnegative
and sum to one.
\end{proof}

\begin{proposition}[Resolution of graph-distance ties]
\label{prop:distance-tie-resolution}
Consider two observed transitions $e_1=(s,a_1,s'_1)$ and
$e_2=(s,a_2,s'_2)$ in the same task--source-state normalization group. If
GraphGPO assigns them equal step advantages, then
\begin{equation}
 B(e_1)-B(e_2)=w_{\mathrm{step}}\eta\left[
 \widehat A^{\mathrm{mix}}(e_1)-
 \widehat A^{\mathrm{mix}}(e_2)\right].
 \label{eq:proof-tie-difference}
\end{equation}
Thus, for $w_{\mathrm{step}}\eta>0$, any strict ordering produced by the
normalized shaped advantage becomes the strict ordering of the MileGPO
step-level credit. In particular, equal destination-to-goal distances imply
equal GraphGPO returns and hence satisfy the premise.
\end{proposition}

\begin{proof}
Applying Equation~\ref{eq:proof-convex-interpolation} to the two transitions
and taking their difference gives the expression below, where
$\Delta_{\mathrm{G}}=\widehat A^{\mathrm{G}}(e_1)
-\widehat A^{\mathrm{G}}(e_2)$ and
$\Delta_{\mathrm{mix}}=\widehat A^{\mathrm{mix}}(e_1)
-\widehat A^{\mathrm{mix}}(e_2)$:
\begin{equation}
B(e_1)-B(e_2)=(1-\eta)w_{\mathrm{step}}\Delta_{\mathrm{G}}
+\eta w_{\mathrm{step}}\Delta_{\mathrm{mix}}.
\end{equation}
The first difference is zero by the premise, which proves
Equation~\ref{eq:proof-tie-difference}. Moreover, the GraphGPO return is
$c\gamma_G^{d(s',g)}$; therefore
$d(s'_1,g)=d(s'_2,g)$ gives equal raw graph returns. Applying the same
normalization to equal values in the same group preserves equality.
\end{proof}

\begin{proposition}[Bounded raw shaping correction]
\label{prop:bounded-shaping}
Assume the task-wise normalized candidate scores lie in $[0,1]$,
$\omega,\gamma_{\Phi}\in(0,1]$, and
$c,\lambda,w_{+},w_{-}\geq0$. Then the one-sided potential increments obey
\begin{equation}
 0\leq\delta_t^{+}\leq w_{+},\qquad
 0\leq\delta_t^{-}\leq w_{-},
 \label{eq:proof-increment-bounds}
\end{equation}
and the raw MileGPO return differs from the GraphGPO return by at most
\begin{equation}
 -c\lambda w_{-}\leq r_t^{\mathrm{M}}-r_t^{\mathrm{G}}
 \leq c\lambda w_{+}.
 \label{eq:proof-return-bound}
\end{equation}
For the reported values $c=10$, $\lambda=0.25$, $w_{+}=1$, and
$w_{-}=0.25$, the correction lies in $[-0.625,2.5]$.
\end{proposition}

\begin{proof}
Each positive or negative potential is the maximum of terms of the form
$\overline S(v)\omega^{d(s,v)}$. The normalized score, distance decay, and their product all lie in $[0,1]$, so
$0\leq\Phi^{+}(s),\Phi^{-}(s)\leq1$. From
Equations~\ref{eq:positive-shaping} and~\ref{eq:negative-shaping},
\begin{equation}
\begin{array}{rcl}
0\leq\delta_t^{+}
&\leq&\gamma_{\Phi}w_{+}\Phi^{+}(s_{t+1})\leq w_{+},\\
0\leq\delta_t^{-}
&\leq&\gamma_{\Phi}w_{-}\Phi^{-}(s_{t+1})\leq w_{-}.
\end{array}
\end{equation}
Finally, Equation~\ref{eq:milegpo-return} gives
$r_t^{\mathrm{M}}-r_t^{\mathrm{G}}
=c\lambda(\delta_t^{+}-\delta_t^{-})$. Substituting the increment bounds
proves Equation~\ref{eq:proof-return-bound}; substituting the reported
hyperparameters gives the stated numerical interval.
\end{proof}

\begin{figure*}[t]
\centering
\includegraphics[width=\textwidth]{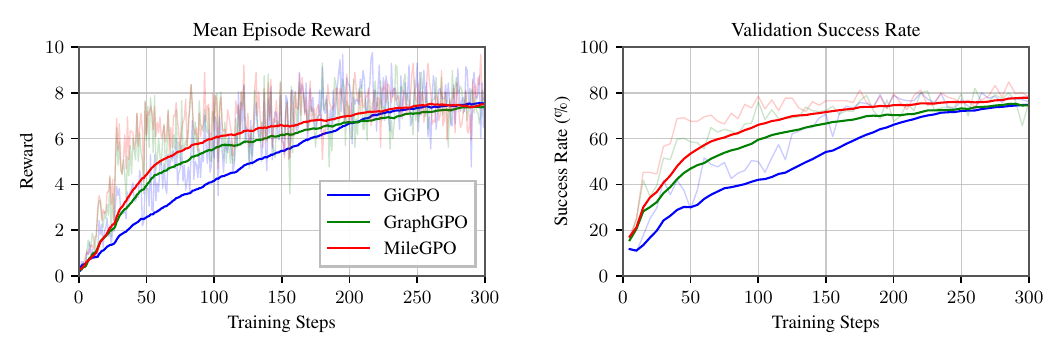}
\caption{WebShop training dynamics of the aligned GiGPO, GraphGPO, and
MileGPO runs. From left to right, the panels show mean episode reward and
validation success rate. Faint lines show per-checkpoint measurements, and
solid lines show the smoothed training trends.}
\label{fig:webshop-training-curves}
\end{figure*}

\subsection{Full WebShop Training Dynamics}
\label{app:webshop-training-dynamics}

Together, these propositions show that MileGPO adds a tunable, tie-resolving,
and bounded correction to graph-distance credit on each sampled rollout graph.

Figure~\ref{fig:webshop-training-curves} reports the complete optimization
trajectories of the aligned GiGPO, GraphGPO, and MileGPO WebShop runs. MileGPO
improves more rapidly during the early and middle stages in both the training
and validation views, while the curves become closer near the end of training.
These trajectories characterize optimization dynamics rather than final
performance; the endpoint comparisons in the main paper follow the independent
evaluation protocol described in Section~\ref{sec:experiments}.

\FloatBarrier
\section{Additional Experimental Analyses}
\label{app:additional-analyses}

\subsection{Checkpoint-Wise Replay of Distance Ties}
\label{app:distance-tie-replay}

Figure~\ref{fig:rollout-md-rcs-replay} expands the pooled replay result in
Figure~\ref{fig:pcc-replay-composite} by resolving it across checkpoints. We
identify sibling transitions that share a parent state, have opposite episode
outcomes, and are tied under GraphGPO's final-state-distance advantage. On the
same GraphGPO trajectories, RCS ranks the successful transition above
the failed one for 26.5\% of the 569 eligible pairs, compared with 17.0\% for
uniform MD. Figure~\ref{fig:rollout-md-rcs-replay} shows where these
corrections occur across checkpoints and also reveals that the number of
eligible ties decreases later in training. Because
the comparison replays alternative shaping on fixed trajectories, it should be
interpreted as descriptive mechanism evidence rather than a causal training
ablation.

We group transitions by task and canonicalized parent state, label a transition
as successful when its trajectory-level episode
return is positive, and declare a GraphGPO advantage tie at an absolute
difference of at most $10^{-6}$. Both replay variants use the same RCS-mined
candidate set. Uniform MD assigns unit weight to every mined milestone and
trap, whereas RCS retains the mined reliability weights; PCC is disabled in
both variants. The replay uses $\gamma_G=0.20$, $\omega=0.20$, and
$\lambda=0.25$, matching the reported WebShop recipe.

\begin{figure}[t]
\centering
\includegraphics[width=\columnwidth]{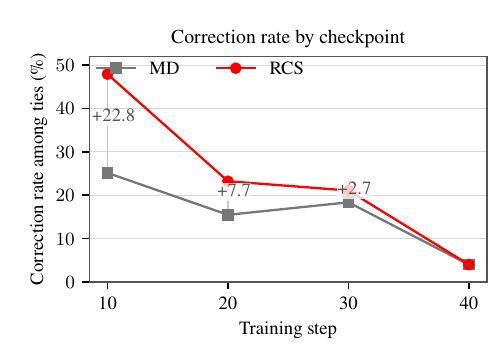}
\caption{Offline replay on the same GraphGPO WebShop rollout traces We compare success/failure
transitions sharing a parent state and tied under GraphGPO's final-state-distance
advantage. We report the correction rate,
$N_{\rm corrected}/N_{\rm tie}$, at each checkpoint with eligible pairs; a
correction means that replayed shaping orders the successful transition above
the failed one. Step 50 is omitted because it contains no eligible ties. This is a
descriptive offline replay, not a causal intervention or a retrained ablation.}
\label{fig:rollout-md-rcs-replay}
\end{figure}

\clearpage
\begin{table*}[t]
\centering
\small
\setlength{\tabcolsep}{1mm}
\resizebox{\textwidth}{!}{%
\begin{tabular}{@{}>{\raggedright\arraybackslash}p{0.10\textwidth}>{\centering\arraybackslash}p{0.02\textwidth}>{\raggedright\arraybackslash}p{0.41\textwidth}>{\raggedright\arraybackslash}p{0.37\textwidth}>{\centering\arraybackslash}p{0.025\textwidth}@{}}
\toprule
\textbf{Branch} & \textbf{$t$} & \textbf{Agent action} & \textbf{Environment feedback / observation} & \textbf{$r_t$} \\
\midrule
Successful & 0 & \texttt{search[machine wash men's dress shirts with polyester heathers, heathers cotton, cotton heather, needle sleeve, classic fit with color: kelly green, and fit type: youth, and size: x-small, and price lower than 40.00 dollars]}
& Results page 1 (50 products); \texttt{B09P39QN2W}, priced at $\$16.99$, is listed. & 0 \\
& 1 & \texttt{click[b09p39qn2w]} & Product page for \texttt{B09P39QN2W}; fit-type, color, and size options are shown. & 0 \\
& 2 & \texttt{click[kelly green]} & Confirms \texttt{kelly green}; product page remains open. & 0 \\
& 3 & \texttt{click[x-small]} & Confirms \texttt{x-small}; selected options are kelly green and x-small. & 0 \\
& 4 & \texttt{click[classic fit]} & Product page unchanged; the requested attribute is not an offered button. & 0 \\
& 5 & \texttt{click[needle sleeve]} & Product page unchanged; the requested attribute is not an offered button. & 0 \\
& 6 & \texttt{click[classic fit]} & Product page unchanged. & 0 \\
& 7 & \texttt{click[youth]} & Confirms \texttt{youth}; selected options are youth, kelly green, and x-small. & 0 \\
& 8 & \texttt{click[x-small]} & Product page unchanged; x-small remains selected. & 0 \\
& 9 & \texttt{click[polyester]; click[heathers]; click[cotton]; click[needle sleeve]; click[classic fit]; click[men]; click[price lower than 40.00 dollars]} & Product page unchanged; the three valid selections remain active. & 0 \\
& 10 & \texttt{click[x-small]} & Product page unchanged. & 0 \\
& 11 & \texttt{click[needle sleeve]} & Product page unchanged. & 0 \\
& 12 & \texttt{click[needle sleeve]} & Product page unchanged. & 0 \\
& 13 & \texttt{click[classic fit]; click[needle sleeve]} & Product page unchanged. & 0 \\
& 14 & \texttt{click[kelly green]; click[40.00]} & Product page unchanged. & 0 \\
& 15 & \texttt{click[polyester heathers]; click[heathers cotton]; click[cotton heathers]; click[lower than 40.00]} & Product page unchanged. & 0 \\
& 16 & \texttt{click[polyester heathers]} & Product page unchanged. & 0 \\
& 17 & \texttt{click[x-small]} & Product page unchanged. & 0 \\
& 18 & \texttt{click[classic fit]; click[price lower than 40.00]} & Product page unchanged. & 0 \\
& 19 & \texttt{click[x-small]} & Product page unchanged. & 0 \\
& 20 & \texttt{click[classic fit]} & Product page unchanged. & 0 \\
& 21 & \texttt{click[x-small]} & Product page unchanged. & 0 \\
& 22 & \texttt{click[kelly green]} & Product page unchanged. & 0 \\
& 23 & \texttt{click[x-small]} & Product page unchanged. & 0 \\
& 24 & \texttt{click[classic fit]} & Product page unchanged. & 0 \\
& 25 & \texttt{click[kelly green]} & Product page unchanged. & 0 \\
& 26 & \texttt{click[youth]} & Product page unchanged. & 0 \\
& 27 & \texttt{click[buy now]} & \texttt{success}; purchase accepted. & 10 \\
\midrule
Failed & 0 & \texttt{search[kelly green, machine wash, polyester heathers, heathers cotton, cotton heathers, needle sleeve, classic fit, youth, x-small, price: \$40.00--\$50.00]} & Results page 1 (50 products); \texttt{B09P39QN2W}, priced at $\$16.99$, is listed. & 0 \\
& 1 & \texttt{click[b09p39qn2w]} & Product page for \texttt{B09P39QN2W}; fit-type, color, and size options are shown. & 0 \\
& 2 & \texttt{click[kelly green]} & Confirms \texttt{kelly green}; product page remains open. & 0 \\
& 3 & \texttt{click[x-small]} & Confirms \texttt{x-small}; selected options are kelly green and x-small. & 0 \\
& 4 & \texttt{click[buy now]} & Purchase response reports task score $0.9$ but no binary success; episode terminates as failure. & 0 \\
\bottomrule
\end{tabular}%
}
\caption{Complete paired WebShop rollout for the representative
price-constraint correction at training step 30. Both branches share the task
and initial observation; all 28 successful and five failed transitions are
shown.}
\label{tab:webshop-qualitative-traces}
\end{table*}
\clearpage

\subsection{Representative Credit-Assignment Corrections}
\label{app:webshop-qualitative-traces}

Table~\ref{tab:webshop-qualitative-traces} makes the aggregate tie correction
in Figure~\ref{fig:rollout-md-rcs-replay} concrete through a complete paired
trace. The two rollouts start from the \emph{same task and exactly the same
initial observation}, but issue different search actions. The successful branch
preserves the requested price ceiling, whereas the failed branch searches in an
incompatible price interval. GraphGPO assigns the two first transitions the
same advantage ($\Delta_{\mathrm G}=0.000$), while RCS favors the successful
branch by $\Delta_{\mathrm{RCS}}=+1.533$. Every environment transition in both
rollouts is shown. The feedback column reports state changes and denotes
repeated product-page observations as ``product page unchanged.''
This case illustrates the complementary roles of MD and RCS. MD exposes
intermediate states as candidate credit anchors, allowing the successful and
failed branches to receive distinct intermediate shaping before terminal
success is observed. RCS then calibrates these anchors by their outcome
support, assigning higher credit to the action that preserves the task
constraint. Across all eligible ties, this calibration raises the correction
rate from 17.0\% with uniform MD to 26.5\%.

Two additional same-state corrections exhibit the same pattern. On a product
page, choosing the listed option \texttt{click[xnj-tshirt342-black]} instead of
emitting text outside the action grammar changes the RCS margin by $+0.803$.
On a compatible loafer page, selecting the requested size with
\texttt{click[12]} instead of retreating with \texttt{click[< prev]} changes it
by $+0.562$.

PCC is disabled in this replay and is therefore not illustrated by these paired
traces. Its contribution is evaluated separately by the cumulative ablation in
Table~\ref{tab:cumulative-ablation} and the candidate-evidence diagnostics in
Figure~\ref{fig:pcc-replay-composite}. Adding PCC to RCS increases WebShop
success from 77.2\% to 78.6\% and task score from 87.9\% to 90.3\%, while
the diagnostics show how local progress and same-parent branch evidence are
used to retain, strengthen, or suppress candidate milestones.

\end{document}